\documentclass{article}

\PassOptionsToPackage{numbers, compress}{natbib}
\usepackage[preprint]{neurips_2023}

\usepackage[utf8]{inputenc} 
\usepackage[T1]{fontenc}    
\usepackage{hyperref}       
\usepackage{url}            
\usepackage{booktabs}       
\usepackage{amsfonts}       
\usepackage{nicefrac}       
\usepackage{microtype}      

\usepackage{lipsum}
\newcommand\blfootnote[1]{%
  \begingroup
  \renewcommand\thefootnote{}\footnote{#1}%
  \addtocounter{footnote}{-1}%
  \endgroup
}
\usepackage{float}
\usepackage{wrapfig}
\usepackage{xspace}
\usepackage{enumitem}
\usepackage{graphicx}
\usepackage{amsthm,amsmath,amssymb}
\DeclareMathOperator{\E}{\mathbb{E}}

\DeclareMathOperator{\Var}{Var}

\title{\discount: Counting in Large Image Collections with Detector-Based Importance Sampling}

\author{%
  Gustavo~Perez \hspace{25pt} Subhransu Maji$^*$ \hspace{25pt} Daniel Sheldon$^*$\\
  Manning College of Information \& Computer Sciences\\
  University of Massachusetts, Amherst\\
  \texttt{\{gperezsarabi, smaji, sheldon\}@cs.umass.edu} \\
}

\usepackage[x11names,table]{xcolor}

\newcommand{\discount}{\textsc{DISCount}\xspace}

\newcommand{\Unif}{\text{Unif}}
\DeclareMathOperator{\I}{\mathbb{I}}
\newcommand{\MC}{\text{MC}}
\newcommand{\IS}{\text{IS}}
\newcommand{\DIS}{\text{DIS}}
\newcommand{\kDIS}{k\text{DIS}}

\newcommand{\CAL}{\text{CAL}}
\newcommand{\toP}{\stackrel{P}{\longrightarrow}}
\newcommand{\toD}{\stackrel{D}{\longrightarrow}}

\begin{document}

\maketitle

\begin{abstract}

Many modern applications use computer vision to detect and count objects in massive image collections. However, when the detection task is very difficult or in the presence of domain shifts, the counts may be inaccurate even with significant investments in training data and model development. We propose \discount -- a detector-based importance sampling framework for counting in large image collections that integrates an imperfect detector with human-in-the-loop screening to produce unbiased estimates of counts.  
We propose techniques for solving counting problems over multiple spatial or temporal regions using a small number of screened samples and estimate confidence intervals. This enables end-users to stop screening when estimates are sufficiently accurate, which is often the goal in a scientific study. On the technical side we develop variance reduction techniques based on control variates and prove the (conditional) unbiasedness of the estimators. 
\discount leads to a 9-12$\times$ reduction in the labeling costs over naive screening for tasks we consider, such as counting birds in radar imagery or estimating damaged buildings in satellite imagery, and also surpasses alternative covariate-based screening approaches in efficiency.
\blfootnote{$^*$equal advising contribution}
\end{abstract}

\section{Introduction}\label{sec:introduction}

Many modern applications use computer vision to detect and count objects in massive image collections. For example, we are interested in applications that involve counting bird roosts in radar images and damaged buildings in satellite images. The image collections are too massive for humans to solve these tasks in the available time. Therefore, a common approach is to train a computer vision detection model and run it exhaustively on the images.

The task is interesting because the goal is not to generalize, but to achieve the scientific counting goal with sufficient accuracy for a \emph{fixed} image collection. The best use of human effort is unclear: it could be used for model development, labeling training data, or even directly solving the counting task!
A particular challenge occurs when the detection task is very difficult, so the accuracy of counts made on the entire collection is questionable even with huge investments in training data and model development.
Some works resort to human screening of the detector outputs~\cite{norouzzadeh2018,nurkarim2023,perez2022}, 
which saves time compared to manual counting but is still very labor intensive.

These considerations motivate \emph{statistical} approaches to counting. Instead of screening the detector outputs for all images, a human can ``spot-check'' some images to estimate accuracy, and, more importantly, use statistical techniques to obtain unbiased estimates of counts across unscreened images. In a related context, \citet{iscount} proposed IS-count, which uses importance sampling to estimate total counts across a collection when (satellite) images are expensive to obtain by using spatial covariates to sample a subset of images.

We contribute counting methods for large image collections that build on IS-count in several ways. 
First, we work in a different model where images are freely available and it is possible to train a detector to run on all images, but the detector is not reliable enough for the final counting task, or its reliability is unknown. We contribute human-in-the-loop methods for count estimation using the detector to construct a proposal distribution, as seen in Fig.~\ref{fig:splash}. 
Second, we consider solving multiple counting problems---for example, over disjoint or overlapping spatial or temporal regions---simultaneously, which is very common in practice. We contribute a novel sampling approach to obtain simultaneous estimates, prove their (conditional) unbiasedness, and show that the approach allocates samples to regions in a way that approximates the optimal allocation for minimizing variance.
Third, we design confidence intervals, which are important practically to know how much human effort is needed. Fourth, we use variance reduction techniques based on control variates.  

Our method produces unbiased estimates and confidence intervals with reduced error compared to covariate-based methods. In addition, the labeling effort is further reduced with \discount as we only have to verify detector predictions instead of producing annotations from scratch. 
On our tasks, \discount leads to a 9-12$\times$ reduction in the labeling costs over naive screening and 6-8$\times$ reduction over IS-Count.
Finally, we show that solving multiple counting problems jointly can be done more efficiently than solving them separately, demonstrating a more efficient use of samples.

\begin{figure}
    \centering
    \includegraphics[width=\linewidth]{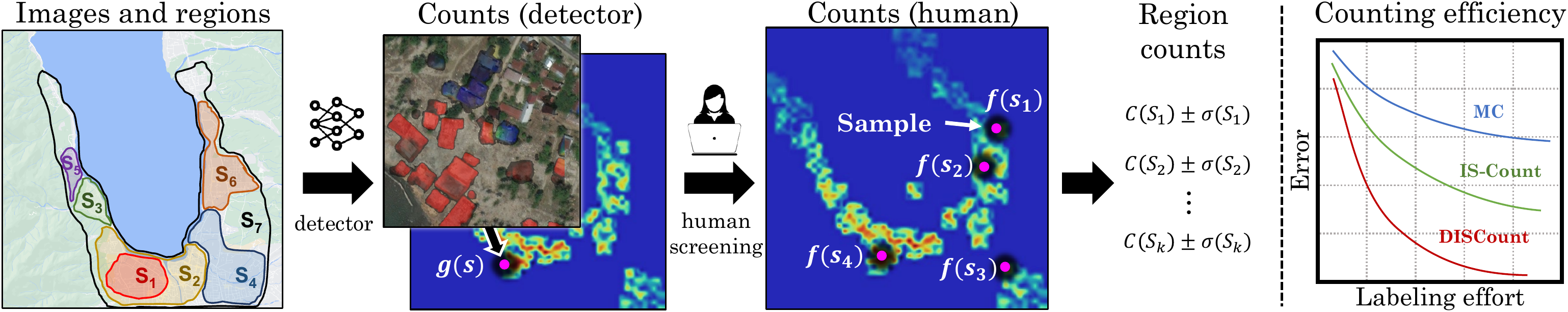}\\
    \hspace{5pt} {\small \textbf{(a)}} \hspace{80pt} {\small \textbf{(b)}} \hspace{80pt} {\small \textbf{(c)}} \hspace{120pt} {\small \textbf{(d)}}
    \caption{\textbf{$k$-\discount} uses detector-based importance sampling to screen counts and solve multiple counting problems. \textbf{(a)} Geographical regions where we want to estimate counts of damaged buildings. \textbf{(b)} Outputs of a damaged building detector on satellite imagery, which can be used to estimate counts $g(s)$ for each tile (shows as dots). \textbf{(c)} Tiles selected for human screening to obtain true counts $f(s)$, from which counts for all regions are joinly estimated by $k$-\discount. \textbf{(d)} Our experiments show that \discount outperforms naive (MC) and covariate-based sampling (IS-Count~\cite{iscount}).
    }
    \label{fig:splash}
\end{figure}

\section{Related Work}\label{sec:relatedwork}

Computer vision techniques have been deployed for counting in numerous applications where exhaustive human-labeling is expensive due to the sheer volume of imagery involved. This includes areas such as detecting animals in camera trap imagery~\cite{norouzzadeh2018,tuia2022perspectives}, counting buildings, cars, and other structures in satellite images~\cite{nurkarim2023,cavender2022integrating,burke2021using,leitloff2010vehicle}, species monitoring in citizen science platforms~\cite{tuia2022perspectives,van2018inaturalist}, monitoring traffic in videos~\cite{won2020intelligent,coifman1998real}, as well as various medicine, science and engineering applications.
For many applications the cost associated with training an \emph{accurate} model is considerably less than that of meticulously labeling the entire dataset. Even with a less accurate model, human-in-the-loop recognition strategies have been proposed to reduce annotation costs by integrating human validation with noisy predictions~\cite{branson2010visual,wah2014similarity}.

Our approach is related to work in active learning~\cite{settles2009active} and semi-supervised learning~\cite{chapelle2009semi}, where the goal is to reduce human labeling effort to learn models that generalize on i.i.d. held out data. 
While these approaches reduce the cost of labels on training data, they often rely on large labeled test sets to estimate the performance of the model, which can be impractical. 
Active testing~\cite{nguyen2018,kossen2021} aims to reduce the cost of model evaluation by providing a statistical estimate of the performance using a small number of labeled examples.
Unlike traditional learning where the goal is performance on held out data, the goal of active testing is to estimate performance on a \emph{fixed} dataset.
Similarly, our goal is to estimate the counts on a fixed dataset, but different from active testing we are interested in estimates of the true counts and not the model’s performance. 
In particular, we want unbiased estimates of counts even when the detector is unreliable. 
Importantly, since generalization is not the goal, overfitting to the dataset statistics may lead to more accurate estimates.

Statistical estimation has been widely used to conduct surveys (e.g., estimating population demographics, polling, etc.)~\cite{cochran1977sampling}. 
In IS-Count~\cite{iscount}, the authors propose an importance sampling approach to estimate counts in large image collections using humans-in-the-loop.
They showed that one can count the number of buildings at the continental scale by sampling a small number of regions based on covariates such as population density and annotating those regions, thereby reducing the cost of obtaining high-resolution satellite imagery and human labels.
However, for many applications the dataset is readily available, and running the detector is cost effective, but human screening is expensive.
To address this, we propose using the detector to guide the screening process and demonstrate that this significantly reduces error rates in count estimation given a fixed amount of human effort.
Furthermore, for some applications, screening the outputs of a detector can be significantly faster than to annotate from scratch, leading to additional savings. 

An interesting question is what is the best way to utilize human screening effort to count on a dataset. For example, labels might be used to improve the detector, measure performance on the deployed dataset, or, as is the case in our work, to derive a statistical estimate of the counts. Our work is motivated by problems where improving the detector might require significant effort, but counts from the detector are correlated with true counts and can be used as a proposal distribution for sampling.

\section{\discount: Detector-based IS-Count}\label{sec:method}

Consider a counting problem in a discrete domain $\Omega$ (usually spatiotemporal) with elements $s \in \Omega$  that represent a single unit such as an image, grid cell, or day of year.
For each $s$ there is a ground truth ``count'' $f(s) \geq 0$, which can be any non-negative measurement, such as the number or total size of all objects in an image.
A human can label the underlying images for any $s$ to obtain $f(s)$. 

Define $F(S) = \sum_{s \in S} f(s)$ to be the cumulative count for a region $S$.
We wish to estimate the total counts $F(S_1), \ldots, F(S_k)$ for $k$ different subsets $S_1, \ldots, S_k \subseteq \Omega$, or \emph{regions}, while using human effort as efficiently as possible.
The regions represent different geographic divisions or time ranges and may overlap --- for example, in the roost detection problem we want to estimate \emph{cumulative} counts of birds for each day of the year, while disaster-relief planners want to estimate building damage across different geographical units such as towns, counties, and states.
Assume without loss of generality that $\bigcup_{i=1}^k S_i = \Omega$, otherwise the domain can be restricted so this is true.

We will next present our methods; derivations and proofs of all results are found in the appendix.

\subsection{Single-Region Estimators}

Consider first the problem of estimating the total count $F(S)$ for a single region $S$. \citet{iscount} studied this problem in the context of satellite imagery, with the goal of minimizing the cost of purchasing satellite images to obtain an accurate estimate.

\paragraph{Simple Monte Carlo~\cite{iscount}} This is a baseline based on simple Monte Carlo sampling. Write $F(S) = \sum_{s \in S}f(s) = |S| \cdot \E_{s \sim \text{Unif}(S)}[f(s)]$. Then the following estimator, which draws $n$ random samples uniformly in $S$ to estimate the total, is unbiased:
\[
    \hat F_{\MC}(S) = |S| \cdot \frac{1}{n} \sum_{i=1}^nf(s_i), \quad s_i \sim \Unif(S).
\]
\paragraph{IS-Count~\cite{iscount}} \citeauthor{iscount} then proposed an estimator based on importance sampling~\cite{mcbook}. Instead of sampling uniformly, the method samples from a proposal distribution $q$ that is cheap to compute for all $s \in S$. For example, to count buildings in US satellite imagery, the proposal distribution could use maps of artificial light intensity, which are freely available. The importance sampling estimator is:
\[
\hat F_{\IS}(S) = \frac{1}{n} \sum_{i=1}^n \frac{f(s_i)}{q(s_i)}, \quad  s_i \sim q.
\]
\paragraph{\discount} IS-count assumes \emph{images} are costly to obtain, which motivates using external covariates for the proposal distribution. 
However, in many scientific tasks, the images are readily available, and the key cost is that of human supervision.
In this case it is possible to train a detection model and run it on all images to produce an approximate count $g(s)$ for each $s$. 
Define $G(S) = \sum_{s \in S}g(s)$ to be the approximate detector-based count for region $S$. 
We propose the \emph{detector-based IS-count} ("\discount{}") estimator, which uses the proposal distribution proportional to $g$ on region $S$, i.e., with density $\bar g_S(s) = g(s)\I[s \in S]/G(S)$. 
The importance-sampling estimator then specializes to:
\[
\hat F_{\DIS}(S) = G(S) \cdot \frac{1}{n} \sum_{i=1}^n \frac{f(s_i)}{g(s_i)}, \quad  s_i \sim \bar g_S.
\]
To interpret \discount, let $w_i = f(s_i)/g(s_i)$ be the ratio of the true count to the detector-based count for the $i$th sample $s_i$ or (importance) \emph{weight}. \discount reweights the detector-based total count $G(S)$ by the average weight $\bar w = \frac{1}{n} \sum_{i=1}^n w_i$, which can be viewed as a correction factor based on the tendency to over- or under-count, on average, across all of $S$. 

\discount is unbiased as long as $\bar g(s) > 0$ for all $s \in S$ such that $f(s) > 0$. Henceforth, we assume detector counts are pre-processed if needed so that $g(s) > 0$ for all relevant units, for example, by adding a small amount to each count.

\subsection{$k$-\discount}\label{sec:kdiscount}

We now return to the multiple region counting problem. A naive approach would be to run \discount separately for each region.
However, this is suboptimal.
First, it allocates samples equally to each region, regardless of their size or predicted count.
Intuitively, we want to allocate more effort to regions with higher predicted counts.
Second, if regions overlap it is wasteful to repeatedly draw samples from each one to solve the estimation problems separately.

\paragraph{$k$-\discount} 
We propose estimators based on $n$ samples drawn from all of $\Omega$ with probability proportional to $g$.
Then, we can estimate $F(S)$ for any region using only the samples from $S$. Specifically, the $k$-\discount estimator is 
\[
\hat F_{\kDIS}(S) = \begin{cases} 
G(S) \cdot \bar w(S) & n(S) > 0 \\[5pt]
0 & n(S) = 0
\end{cases}
, \qquad s_i \sim \bar g_{\Omega},
\]
where $n(S) = |\{i: s_i \in S\}|$ is the number of samples in region $S$ and $\bar{w}(S) = \frac{1}{n(S)} \sum_{i: s_i \in S} w_i$ is the average importance weight for region $S$. 

\newtheorem{claim}{Claim}
\begin{claim}
\label{claim:conditional-unbiasedness}
The $k$-\discount estimator $\hat F_{\kDIS}(S)$ is conditionally unbiased given at least one sample in region $S$. That is, $\E[\hat F_{\kDIS}(S) \mid n(S) > 0] = F(S)$.
\end{claim}
The unconditional bias can also be analyzed (see Appendix). Overall, bias has negligible practical impact. 
It occurs only when the sample size $n(S)$ is zero, which is an event that is both observable and has probability $(1-p(S))^n$ that decays exponentially in $n$, where $p(S)=G(S)/G(\Omega)$.

In terms of variance, $k$-\discount behaves similarly to \discount run on each region $S$ with sample size equal to $\E[n(S)] = n p(S)$. 
To first order, both approaches have variance $\frac{G(S)^2 \cdot \sigma^2(S)}{n p(S)}$ where $\sigma^2(S)$ is the importance-weight variance.
In the case of \emph{disjoint} regions, running \discount on each region is the same as \emph{stratified importance sampling} across the regions, and the allocation of $n p(S)$ samples to region $S$ is optimal in the following sense:
\begin{claim}
\label{claim:optimal-sample-allocation}
Suppose $S_1, \ldots, S_k$ partition $\Omega$ and the importance weight variance $\sigma^2(S_i) = \sigma^2$ is constant across regions. Assume \discount is run on each region $S_i$ with $n_i$ samples. Given a total budget of $n$ samples, the sample sizes that minimize $\sum_{i=1}^k \Var(\hat F_{\DIS}(S_i))$ are given by $n_i  = n p(S_i) = n G(S_i)/G(\Omega)$.
\end{claim}
The analysis uses reasoning similar to the \emph{Neyman allocation} for stratified sampling~\cite{cochran1977sampling}, and shows that $k$-\discount approximates the optimal allocation of samples to (disjoint) regions under the stated assumptions.
One key difference is that $k$-\discount draws samples from all of $\Omega$ and then assigns them to regions, which is called ``post-stratification'' in the sampling literature~\cite{cochran1977sampling}.
An exact variance analysis in the Appendix reveals that, if the expected sample size $np(S)$ for a region is very small, $k$-\discount may have up to 30\% ``excess'' variance compared to stratification due to the random sample size, but the excess variance disappears quickly and both approaches have the same asymptotic variance.
A second key difference to stratification is that regions can overlap; $k$-\discount's approach of sampling from all of $\Omega$ and then assigning samples to regions extends cleanly to this setting.

\subsection{Control Variates}\label{sec:cv_ci}

Control variates are functions $h(s)$ whose integrals $H(S)= \sum_{s\in S} h(s)$ are known and can be combined with importance sampling using the following estimator: 
\[
\hat F_{\kDIS cv}(S) = \begin{cases} 
G(S) \cdot \bar w_{h}(S) + H(S) & n(S) > 0 \\[5pt]
0 & n(S) = 0
\end{cases}
, \qquad s_i \sim \bar g_{\Omega},
\]
where $\bar{w}_{h}(S) = \frac{1}{n(S)} \sum_{i: s_i \in S} w_{h,i}$ and $w_{h,i} = (f(s_i) - h(s_i))/g(s_i)$. It is clear that $\hat{F}_{\kDIS cv}(S)$ has the same expectation as $\hat{F}_{\kDIS(S)}$, but $\hat{F}_{\kDIS cv}(S)$ might have a lower variance under certain conditions (if $f$ and $h$ are sufficiently correlated~\cite{mcbook}). 
For bird counting, estimated counts from previous years could be used as control variates as migration is periodic to improve count estimates (see experiments in \S~\ref{sec:exp} for details). 

\subsection{Confidence intervals}\label{sec:confidenceintervals}
Confidence intervals for $k$-\discount can be constructed in a way similar to standard importance sampling. For a region $S$, first estimate the importance weight variance $\sigma^2(S)$ as:
\[
\hat \sigma^2(S) = \frac{1}{n(S)}\sum_{i: s_i \in S}\left(\frac{f(s_i)}{g(s_i)}-\frac{\hat F_{\kDIS}(S)}{G(S)}\right)^2.
\]
An approximate $1-\alpha$ confidence interval is then given by $\hat F_{\kDIS}(S) \pm z_{\alpha/2} \cdot G(S)\cdot\hat \sigma(S)/\sqrt{n(S)}$, where $z_\gamma$ is the $1-\gamma$ quantile of the standard normal distribution, e.g., $z_{0.025} = 1.96$ for a 95\% confidence interval. 
The theoretical justification is subtle due to scaling by the \emph{random} sample size $n(S)$. It is based on the following asymptotic result, proved in the Appendix.
\begin{claim}
\label{claim:asymptotic-distribution}
The $k$-\discount estimator with scaling factor $G(S) \hat \sigma(S)/\sqrt{n(S)}$ is asymptotically normal, that is, the distribution of
$\frac{\hat F_{\kDIS}(S) - F(S)}{G(S) \cdot \hat \sigma(S)/\sqrt{n(S)}}$
converges to $\mathcal N(0, 1)$ as $n \to \infty$.
\end{claim}
In preliminary experiments we observed that for small expected sample sizes the importance weight variance $\sigma^2(S)$ can be underestimated leading to intervals that are too small --- as an alternative, we propose a practical heuristic for smaller sample sizes where $\hat \sigma^2(\Omega)$ is used instead of $\hat\sigma^2(S)$; that is, \emph{all} samples are used to estimate variability of importance weights for each region $S$.

\section{Experimental Setup}\label{sec:exp}

In this section we describe the counting tasks and detection models (\S~\ref{sec:exp:roosts}--\ref{sec:exp:buildings}) and the evaluation metrics (\S~\ref{sec:evalmetrics}) we will use to evaluate different counting methods. 
We focus on two applications: counting roosting birds in weather radar images and counting damaged buildings in satellite images of a region struck by a natural disaster. 

\subsection{Counting Roosting Birds from Weather Radar} \label{sec:exp:roosts}
Many species of birds and bats congregate in large numbers at nighttime or daytime roosting locations.
Their departures from these ``roosts'' are often visible in weather radar, from which it's possible to estimate their numbers~\cite{winkler2006roosts,buler2012mapping,horn2008analyzing}. 
The US ``NEXRAD'' weather radar network~\cite{nexrad} has collected data for 30 years from 143+ stations and provides an unprecedented opportunity to study long-term and wide-scale biological phenomenon such as roosts~\cite{rosenbergDeclineNorthAmerican2019,sanchez2019worldwide}.
However, the sheer volume of radar scans ($>$250M) prevents manual analysis and motivates computer vision approaches~\cite{chilson2018automated,lin2019mistnet,cheng2020,perez2022}.
Unfortunately, the best computer vision models~\cite{perez2022,cheng2020} for detecting roosts have average precision only around 50\% and are not accurate enough for fully automated scientific analysis, despite using state-of-the-art methods such as Faster R-CNNs~\cite{fasterrcnn} and training on thousands of human annotations --- the complexity of the task suggests substantial labeling and model development efforts would be needed to improve accuracy, and may be impractical.

\begin{wrapfigure}{r}{0.5\textwidth}
  \begin{center}
    \includegraphics[width=0.5\textwidth]{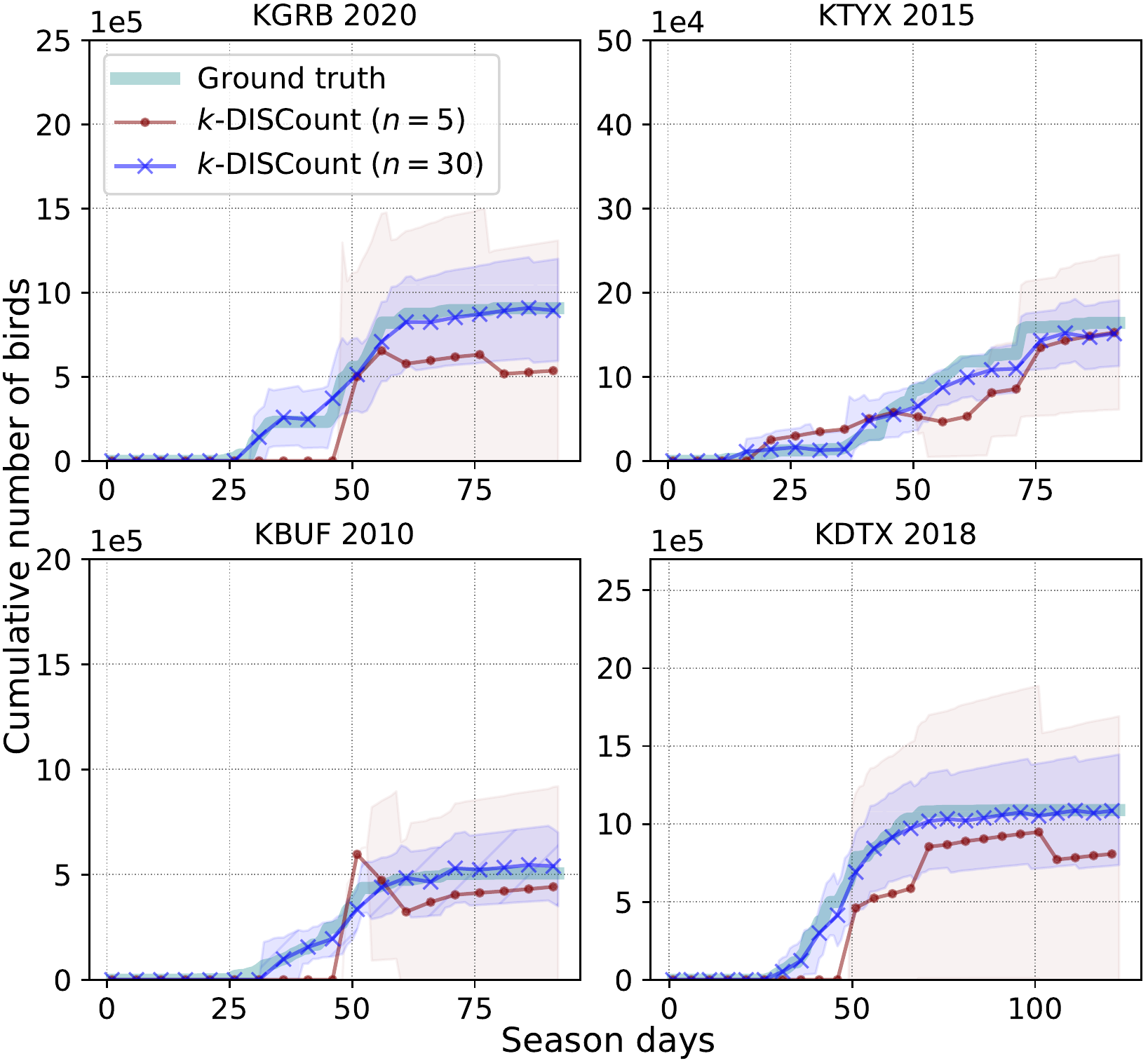}
  \end{center}
  \caption{Count estimates with confidence intervals for two station years (i.e., KGRB 2020 and KBUF 2010) using different numbers of samples.}
  \label{fig:radar_estimates}
\end{wrapfigure}

Previous work~\cite{maria,hermione} used a roost detector combined with manual screening of the detections to analyze more than 600,000 radar scans spanning a dozen stations in the Great Lakes region of the US to reveal patterns of bird migration over two decades. 
The vetting of nearly 64,000 detections was orders of magnitude faster than manual labeling, yet still required a substantial 184 hours of manual effort.
Scaling to the entire US network would require at least an order of magnitude more effort, thus motivating a statistical approach.

We use the exhaustively screened detections from the Great Lakes analysis in~\cite{maria,hermione} to systematically analyze the efficiency of sampling based counting. The data is organized into domains $\Omega^{\texttt{sta,yr}}$ corresponding to 12 stations and 20 years (see Fig.~\ref{fig:radar} in Appendix~\ref{app:countingtasks}). 
Thus the domains are disjoint and treated separately.
Counts are collected for each day $s$ by running the detector using all radar scans for that day to detect and track roost signatures and then mapping detections to bird counts using the measured radar ``reflectivity'' within the tracks.
For the approximate count $g(s)$ we use the automatically detected tracks, while for the true count $f(s)$ we use the manually screened and corrected tracks.
For a single domain, i.e., each station-year, we divide a complete roosting season into temporal regions in three different scenarios: (1) estimating bird counts up to each day in the roosting season (i.e., regions are nested prefixes of days in the entire season), (2) the end of each quarter of (i.e., regions are nested prefixes of quarters in the entire season), and (3) estimating each quarter's count (each region is one quarter). We measure error using the fully-screened data and average errors across all domains and regions.  Fig.~\ref{fig:radar_estimates} shows the counts and confidence intervals estimated using $k$-\discount for the first scenario on four station-years.

\subsection{Counting Damaged Buildings from Satellite Images} \label{sec:exp:buildings}

Building damage assessment from satellite images~\cite{su2018,deng2022} is often used to plan humanitarian response after a natural disaster strikes. However, the performance of computer vision models degrades when applied to new regions and disaster types. Our approach can be used to quickly vet the data produced by the detector to correctly estimate counts in these scenarios.

We use the building damage detection model by~\cite{xview1st}, the winner of the xView2 challenge~\cite{xview2}. The model is based on U-Net~\cite{unet} to detect buildings in the pre-disaster image, followed by a ``siamese network" that incorporates at pre- and post-disaster images to estimate damage.
The model is trained on the xBD dataset~\cite{xbddataset} that contains building and damage annotations spanning multiple geographical regions and disaster types (e.g., earthquake, hurricane, tsunami, etc.). 
While the dataset contains four levels of damage (i.e., 0: no-damage, 1: minor-damage, 2: major-damage, and 3: destroyed), in this work we combine all damage levels (i.e., classes 1-3) into a single ``damage'' class. 

We consider the Palu Tsunami from 2018; the data consists of 113 high-resolution satellite images labeled with 31,394 buildings and their damage levels.
We run the model on each tile $s$ to estimate the number of damaged buildings $g(s)$, while the ground-truth number of damaged buildings is used as $f(s)$. 
Our goal is to estimate the cumulative damaged building count in sub-regions expanding from the area with the most damaged buildings as shown in Fig.~\ref{fig:buildings_regions} in the Appendix~\ref{app:regions}.
To define the sub-regions, we sort all $m$ images by their distance from the epicenter (defined as the image tile with the most damaged buildings) and then divide into chunks or ``annuli'' $A_1, \ldots, A_7$ of size $m/7$.
The task is to estimate the cumulative counts $S_j = \bigcup_{i=1}^j A_i$ of the first $j$ chunks for $j$ from 1 to 7.

\subsection{Evaluation}\label{sec:evalmetrics}
We measure the fractional error between the \emph{true} and the estimated counts averaged over all regions in a domain $S_{1}, \ldots, S_{k} \subseteq \Omega$ as: 
\[
\texttt{Error} (\Omega) = \frac{1}{k } \sum_{i=1}^{k} \frac{|F(S_{i}) - \hat F(S_{i})|}{F(\Omega)}.
\]
For the bird counting task, for any given definition of regions within one station-year $\Omega$ (i.e., cumulative days or quarters defined in \S~\ref{sec:exp:roosts}) we report the error averaged across all station-years 
corresponding to 12 stations and $\approx 20$ years. For the damaged building counting problem there is only a single domain corresponding to the Palu Tsunami region.
In addition, we calculate the average confidence interval width normalized by $F(\Omega)$.
We run 1000 trials and plot average metrics $\pm 1.96 \times \text{std. error}$ over the trials.
We also evaluate confidence interval coverage, which is the fraction of confidence intervals that contain the true count over all domains, regions, and trials.

\begin{figure}[t]
    \centering
    \includegraphics[width=1.0\linewidth]{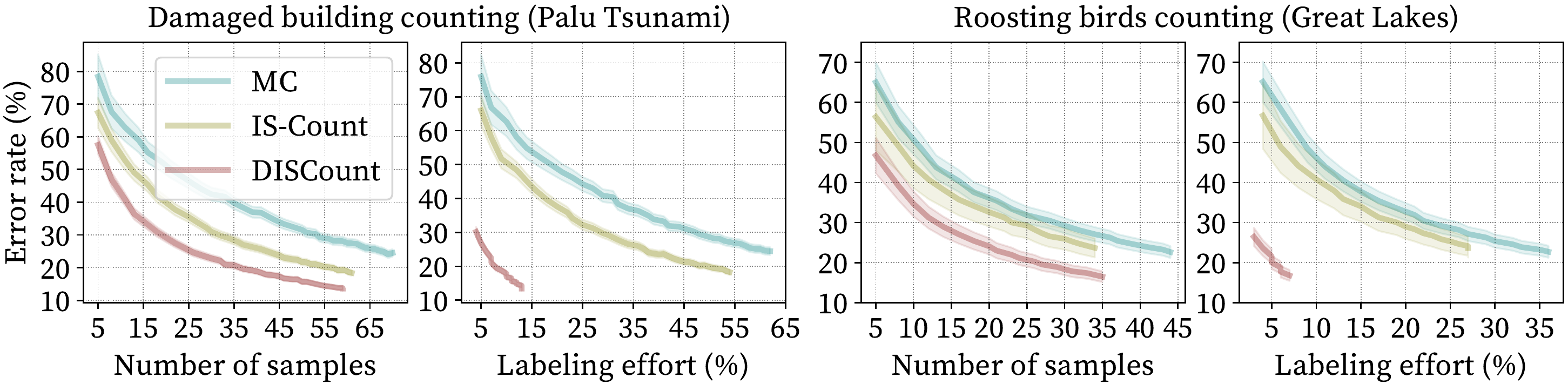}
    \caption{\textbf{Detector-based sampling.} Estimation error of damaged building counts in the Palu Tsunami region from the xBD dataset (left) and counting roosting birds from the Great Lakes radar stations in the US from NEXRAD data (right). We get lower error with \discount compared to IS-Count and simple Monte Carlo sampling (MC).  
    The labeling effort is further reduced with \discount since the user is not required to label an image from scratch but only to verify outputs from the detector (See \S~\ref{sec:res:labelingeffort} for details). The estimation errors are averaged over 1000 runs. 
    }
    \label{fig:is_vs_dis} 
\end{figure}

\section{Results}\label{sec:results}

In this section, we present the results comparing detector-based to covariate-based sampling. 
Also, we show reductions in labeling effort 
and demonstrate the advantages of estimating multiple counts jointly. 
Finally, we show confidence intervals and control variates results.

\paragraph{Detector-based sampling reduces error}\label{sec:res:discount}
We first compare \discount (detector-based sampling) to IS-Count and simple Monte Carlo sampling for estimating $F(\Omega)$, that is, the total counts of birds in a complete roosting season for a given station year, or damaged buildings in the entire disaster region.  
Fig.~\ref{fig:is_vs_dis} shows the error rate as a function of number of labeled samples (i.e., the number of \emph{distinct} $s_i$ sampled, since each $s$ is labeled at most once).
In the buildings application, a sample refers to an image tile of size $1024 \times 1024$ pixels, while for the birds a sample refers to a single day.

Using the detector directly without \emph{any} screening results in high error rates --- roughly 136\%  and 149\% for estimating the total count for the damaged buildings and bird counting tasks respectively. \citet{iscount} show the advantages of using importance sampling with screening to produce count estimates with base covariates as opposed to simple Monte Carlo sampling (MC vs. IS-Count). For the bird counting task, we construct a non-detector covariate $g_{\IS}$ by fitting a spline to $f(s)$ with 10\% of the days from an arbitrarily selected station-year pair (station KBUF in 2001). For the damaged building counting task, the covariate $g_{\IS}$ is the true count of all buildings (independent of the damage) obtained using the labels provided with the xBD dataset.

Covariate-based sampling (IS-Count) leads to significant savings over simple Monte Carlo sampling (MC), but \discount provides further improvements. In particular, to obtain an error rate of 20\% \discount requires $\approx 1.6\times$ fewer samples than IS-Count and $\approx 3\times$ fewer samples than MC for both counting problems. 

\begin{figure}[t]
    \centering
    \includegraphics[width=1.0\linewidth]{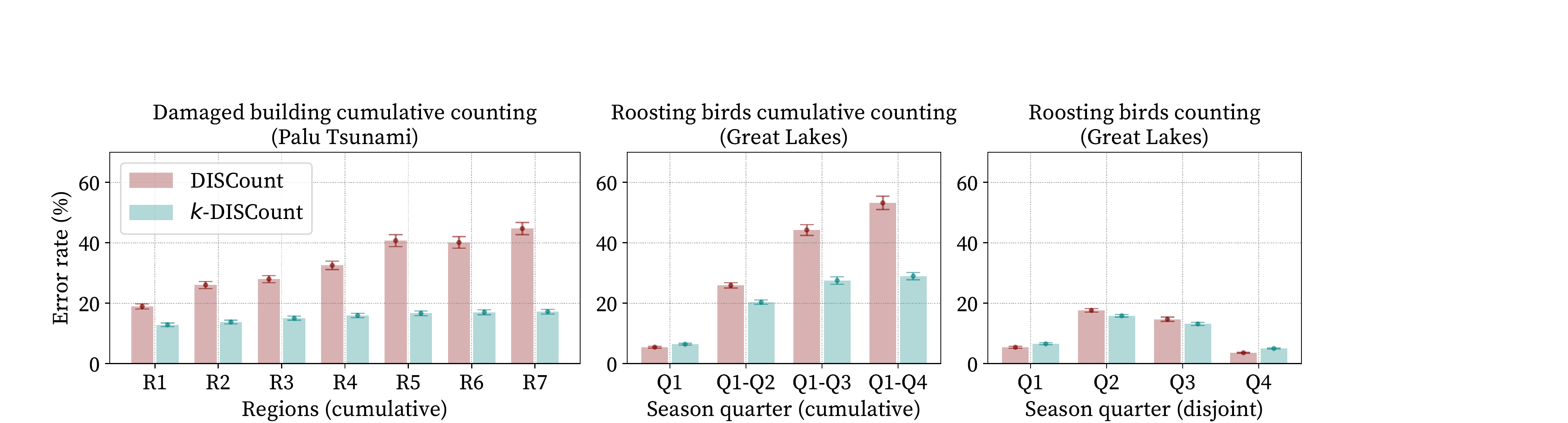}\\
    \caption{\textbf{Solving multiple counting problems jointly.} Estimation error of counting damaged buildings in the Palu Tsunami region from the xBD dataset (left) and counting roosting birds from the Great Lakes radar stations in the US from NEXRAD data (right). 
    We compare solving the counting problems jointly ($k$-\discount) 
    against solving the counting problems separately (\discount). We use 10 samples for both these tests. The estimation errors are averaged over 1000 runs. 
    }
    \vspace{-5pt}
    \label{fig:nested} 
\end{figure}

\paragraph{Screening leads to a further reduction in labeling effort}\label{sec:res:labelingeffort} 
\discount alleviates the need for users to annotate an image from scratch, such as identifying an object and drawing a bounding box around it. Instead, users only need to verify the detector's output, which tends to be a quicker process.
In a study by \citet{su2012} on the ImageNet dataset~\cite{imagenet}, the median time to draw a bounding-box was found to be 25.5 seconds, whereas verification took only 9.0 seconds (this matches the screening time of $\approx$10s per bounding-box in~\cite{hermione,maria}). The right side of Fig.~\ref{fig:is_vs_dis} presents earlier plots with the x-axis scaled based on labeling effort, computed as $100\cdot c\cdot n/|\Omega|$, where $n$ denotes the number of screened samples and $c \in [0,1]$ represents the fraction of time relative to labeling from scratch. For instance, the labeling effort is 100\% when all elements must be labeled from scratch ($c=1$ and $n=|\Omega|$).
For \discount, we estimate $c_{\DIS}=9.0/(25.5 + 9.0)=0.26$, since annotating from scratch requires both drawing and verification, while screening requires only verification. To achieve the same 20\% error rate, \discount requires $6\times$ less effort than IS-Count and $9\times$ less effort than MC for the bird counting task, and $8\times$ less effort than IS-Count and $12\times$ less effort than MC for building counting. 

\paragraph{Multiple counts can be estimated efficiently ($k$-\discount)}\label{sec:res:kdiscount}
To solve multiple counting problems, we compared $k$-\discount to using \discount separately on each region. 
For bird counting, the task was to estimate four quarterly counts (cumulative or individual) as described in \S~\ref{sec:exp:roosts}.
For $k$-\discount, we sampled $n=40$ days from the complete season to estimate the counts simultaneously.
For $\discount$, we solved each of the four problems separately using $n/4 = 10$ samples per region for the same total number of samples.
For building damage counting, the task was to estimate seven cumulative counts as described in \S~\ref{sec:exp:buildings}.
For $k$-\discount, we used $n=70$ images sampled from the entire domain, while for \discount we used $n/7 = 10$ sampled images per region. 

Fig.~\ref{fig:nested} shows that solving multiple counting problems jointly ($k$-\discount) is better than solving them separately (\discount).
For the cumulative tasks, $k$-\discount makes much more effective use of samples from overlapping regions.
For single-quarter bird counts, $k$-\discount has slightly higher error in Q1 and Q4 and lower errors in Q2 and Q3.
This can be understood in terms of sample allocation: $k$-\discount allocates in proportion to predicted counts, which provides more samples and better accuracy in Q2-Q3, when many more roosts appear, and approximates the optimal allocation of Claim~\ref{claim:optimal-sample-allocation}. \discount allocates samples equally, so has slightly lower error for the smaller Q1 and Q4 counts. 
In contrast, for building counting, $k$-\discount has lower error even for the smallest region R1, since this has the most damaged buildings and thus gets more samples than \discount.
Fig.~\ref{fig:ci_cv} (left) shows $k$-\discount outperforms simple Monte Carlo (adapted to multiple regions similarly to $k$-\discount) for estimating cumulative daily bird counts as in  Fig.~\ref{fig:radar_estimates}.

\begin{figure}[t]
    \centering
    \includegraphics[width=1.0\linewidth]{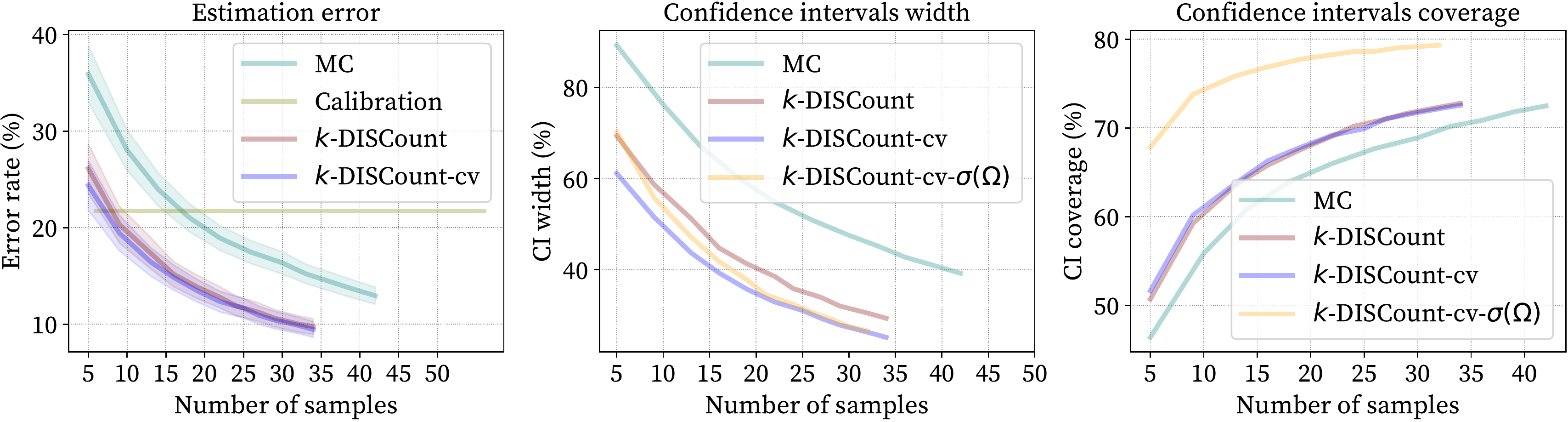}
    \vspace{-10pt}
    \caption{\textbf{Control variates and confidence intervals on bird counting.}  
    We compare simple Monte Carlo (MC), calibration with isotonic regression, and variations of $k$-\discount that include control variates (-cv) and improved variance estimates $\left(-\sigma(\Omega)\right)$.
    (left) Error rates using $k$-\discount are significantly smaller than MC and calibration. (middle) Confidence intervals' width. (right) Confidence intervals' coverage. The error and the confidence intervals' width are slightly reduced when control variates are used while maintaining the coverage. Furthermore, $k$-\discount-cv-$\sigma(\Omega)$ improves the coverage. The results are averaged over all station-years and over 1,000 runs. 
    }
    \label{fig:ci_cv} 
\end{figure}

\paragraph{Confidence intervals}\label{sec:res:confintervals}
We measure the width and coverage of the estimated confidence intervals (CIs) per number of samples for cumulative daily bird counting; see examples in Fig.~\ref{fig:radar_estimates}. 
We compare the CIs of $k$-\discount, $k$-\discount-cv (control variates), $k$-\discount-cv-$\sigma(\Omega)$ (using all samples to estimate variance), and simple Monte Carlo sampling in Fig.~\ref{fig:ci_cv}. When using control variates, the error rate and the CI width are slightly reduced while keeping the same coverage. 
CI coverage is lower than the nominal coverage (95\%) for all methods, but increasing with sample size and substantially improved by $k$-\discount-cv-$\sigma(\Omega)$, which achieves up to $\approx 80\%$ coverage. 
Importance weight distributions can be heavily right-skewed and the variance easily underestimated~\cite{hesterberg1996estimates}.

\paragraph{\discount improves over a calibration baseline}\label{sec:res:calibration}
We implement a calibration baseline where the counts are estimated as  
 $\hat{F}_{\CAL}(S) = \sum_{s \in S} \hat \phi(g(s)),$
where we learn an isotonic regression model $\hat \phi$ between the predicted and true counts trained for each station using 15 uniformly selected samples from one year from that station. Results are shown as the straight line in Fig.~\ref{fig:ci_cv} (left). \discount outperforms calibration with less than 10 samples per station suggesting the difficulties in generalization across years using a simple calibration approach.

\paragraph{Control variates ($k$-\discount-cv)} \label{sec:res:controlvariates}
We perform experiments adding control variates to $k$-\discount in the roosting birds counting problem. 
We use the calibrated detector counts $\hat{\phi}(g(s))$ defined above as the control variate for each station year. 
Fig.~\ref{fig:ci_cv} shows that control variates reduce the confidence interval width (middle: $k$-\discount vs. $k$-\discount-cv) without hurting coverage (right). In addition, the error of the estimate is reduced slightly, as shown in Fig.~\ref{fig:ci_cv} (left). Note that this is achieved with a marginal increase in the labeling effort.

\section{Discussion and Conclusion}\label{sec:discussion}

We contribute methods for counting in large image collections with a detection model.
When the task is complex and the detector is imperfect, allocating human effort to estimate the scientific result directly might be more efficient than improving the detector.
For instance, performance gains from adding more training data may be marginal for a mature model.
Our proposed solution produces accurate and unbiased estimates with a significant reduction in labeling costs from naive and covariate-based screening approaches.
We demonstrate this in two real-world open problems where data screening is still necessary despite large investments in model development. Our approach is limited by the availability of a good detector, and confidence interval coverage is slightly low; possible improvements are to use bootstrapping or corrections based on importance-sampling diagnostics~\cite{hesterberg1996estimates}.

\section{Acknowledgements}

We thank Wenlong Zhao for the deployment of the roost detector, Maria Belotti, Yuting Deng, and our Colorado State University AeroEco Lab collaborators for providing the screened data of the Great Lakes radar stations, and Yunfei Luo for facilitating the building detections on the Palu Tsunami region. This work was supported by the National Science Foundation award \#2017756.

\bibliographystyle{unsrtnat}
\bibliography{neurips_2020}

\clearpage
\appendix

\section{Derivations}

\subsection{IS-Count}

Take $p(s) = 1/|S|$ and $\tilde f(s) = |S| f(s)$, we want $\E_{p}[\tilde f(s)] = \sum_{s \in S} \frac{1}{|S|} |S| f(s) = F(S)$.  Importance sampling with proposal $q$ gives
$$
\begin{aligned}
F(S) 
&= \E_p[\tilde f(s)]  \\[5pt]
&= \E_q\left[\frac{p(s)}{q(s)} \tilde f(s) \right] \\[5pt]
&= \E_q\left[\frac{1/|S|}{q(s)} |S| f(s) \right] \\[5pt]
&= \E_q\left[\frac{f(s)}{q(s)} \right]
\end{aligned}
$$

\subsection{DISCount}
Take $q = \bar g_S$ in IS-Count, then 
$$
\begin{aligned}
F(S) &= \E_q\left[\frac{f(s)}{q(s)} \right] \\[5pt]
&= \E_{\bar g_S}\left[ \frac{f(s)}{g(s)/G(S)} \right] \\[5pt]
&= G(S) \cdot  \E_{\bar g_S}\left[ \frac{f(s)}{g(s)} \right]
\end{aligned}
$$

\subsection{$k$-DISCount}

\begin{proof}[Proof of Claim~\ref{claim:conditional-unbiasedness}]
For any $m > 0$ we have
$$
\begin{aligned}
\E\left[\hat F_{\kDIS}(S) \,\big\vert\, n(S) = m\right]
&= \E\left[ G(S) \cdot \frac{1}{m} \sum_{i=1}^n w_i\cdot \I[s_i \in S] \,\bigg\vert\, n(S) = m \right] \\
&= G(S) \cdot \frac{1}{m} \sum_{i=1}^n \E\left[ w_i \cdot \I[s_i \in S] \,\big\vert\, n(S) = m\right] \\
&= G(S) \cdot \frac{1}{m} \sum_{i=1}^n \Pr[s_i \in S \mid n(S) = m] \cdot \E\left[w_i \,\vert\, s_i \in S, n(S) = m \right] \\
&= G(S) \cdot \frac{1}{m} \sum_{i=1}^n \Pr[s_i \in S \mid n(S) = m] \cdot \E\left[w_i \,\vert\, s_i \in S \right] \\
&= G(S) \cdot \frac{1}{m} \sum_{i=1}^n \frac{m}{n} \cdot \E\left[w_i \,\big\vert\, s_i \in S \right] \\
&= G(S) \cdot \frac{1}{m} \sum_{i=1}^n \frac{m}{n} \cdot \frac{F(S)}{G(S)} \\
&= F(S)
\end{aligned}
$$
In the third line, we used the fact that $\E\big[h(X)\cdot \I[X \in A]\big] = \Pr[x \in A] \cdot \E[h(X) \mid X \in A]$ for any random variable $X$ and event $A$ (see Lemma~\ref{lemma:lemma1} below). In the fourth line we used the fact that $s_i$ is conditionally independent of $n(S)$ given $s_i \in S$, since $n(S) = \I[s_i \in S] + \sum_{j \neq i} \I[s_j \in S]$  and the latter sum is independent of $s_i$. In the fifth line we used the fact that $\Pr[s_i \in S \mid n(S) = m] = \frac{m}{n}$ because $n(S) = \sum_{j=1}^n \I[s_j \in S]$ and the terms in the sum are exchangeable. In the sixth line we computed the conditional expectation as follows using the fact that the conditional density of $s_i$ given $s_i \in S$ is equal to $g(s_i)/G(S)$:
$$
\begin{aligned}
\E\left[w_i \,\big\vert\, s_i \in S \right] 
= \E\left[\frac{f(s_i)}{g(s_i)} \,\bigg\vert\, s_i \in S\right]
= \sum_{s \in S} \frac{g(s_i)}{G(S)} \cdot \frac{f(s_i)}{g(s_i)}
= \frac{1}{G(S)} \sum_{s \in S} f(s) = \frac{F(S)}{G(S)}.
\end{aligned}
$$
\end{proof}

The unconditional bias of $k$-\discount can also be analyzed:
\begin{claim}
\label{claim:bias}
Let $p(S) = \Pr[s_i \in S] = G(S)/G(\Omega)$. The bias of the $k$-\discount estimator is $\E[\hat F_{\normalfont\kDIS}(S)] - F(S) =  - \big(1-p(S)\big)^n$.
\end{claim}
In particular, bias decays exponentially with $n$ and quickly becomes negligible, with magnitude at most $\epsilon$ for $n \geq \log(1/r)/\log(1/\epsilon)$ and $r=1-p(S)$.
Further, the bias is easily computable from the detector counts and therefore known prior to sampling, and the event that leads to a biased estimate $(n(S)=0)$ is observed after sampling. 
All these factors make bias a very minor concern.\footnote{The $k$-\discount estimator can be debiased by dividing by $u = 1-(1-p(S))^n < 1$. However, this leads to higher overall error: if $n(S) = 0$, the estimator is unchanged, and conditioned on the event $n(S) > 0$ the estimator becomes biased and has higher variance by a factor of $1/u^2 > 1$.}

\begin{proof}[Proof of Claim~\ref{claim:bias}]
Using Claim~\ref{claim:conditional-unbiasedness}, we compute the unconditional expectation as
$$
\begin{aligned}
\E\left[\hat F_{\kDIS}(S)\right] 
&= \Pr[n(S) = 0] \cdot \E\left[\hat F_{\kDIS}(S) \mid n(S) = 0\right] + \Pr[n(S) > 0] \cdot \E[\hat F_{\kDIS}(S) \mid n(S) > 0 ] \\[5pt]
&= \Pr[n(S) > 0] \cdot F(S) \\[5pt]
&= \bigg(1 - \bigg(1- \frac{G(S)}{G(\Omega)}\bigg)^n\bigg) \cdot F(S).
\end{aligned}
$$
In the final line, $1 - G(S)/G(\Omega)$ is probability that $s_i \notin S$ for a single $i$, and $(1 - G(S)/G(\Omega))^n = \Pr[n(S) = 0]$ is the probability that $s_i \notin S$ for all $i$. Rearranging gives the result.
\end{proof}

\newtheorem{lemma}{Lemma}
\begin{lemma}
\label{lemma:lemma1}
$\E\big[h(X) \cdot \I[X \in A]\big] = \Pr[x \in A] \cdot \E[h(X) \,|\, X \in A]$ for any random variable $X$ and event $A$.
\end{lemma}
\begin{proof}
Observe
$$
\begin{aligned}
\E\big[h(X) \cdot \I[X \in A]\big]  
&= \sum_x \Pr[X=x] h(x) \I[x \in A] \\
&= \sum_x \Pr[X=x, X \in A] h(x) \\
&= \Pr[X \in A] \sum_x \Pr[X=x \mid X \in A ] h(x) \\
&= \Pr[X \in A] \cdot \E[h(X) \mid X \in A].
\end{aligned}
$$
\end{proof}

\subsection{Optimal allocation of samples for \discount to disjoint regions}

\begin{proof}[Proof of Claim~\ref{claim:optimal-sample-allocation}]
The proof is similar to that of Theorem 5.6 in \cite{cochran1977sampling}.
We prove the claim for $k=2$; the proof generalizes to larger $k$ in an obvious way. The variance of \discount on $S_i$ is
\[
\Var(\hat F_{\DIS}(S_i)) = \frac{G(S_i)^2 \cdot \sigma^2(S_i)}{n_i} = \sigma^2 \frac{G(S_i)^2 }{n_i}.
\]
We want to minimize $\sum_i \Var(\hat F_{\DIS}(S_i))$, which with $k=2$ is proportional to
\[
V = \frac{G(S_1)^2}{n_1} + \frac{G(S_2)^2}{n_2}.
\]
By the Cauchy-Shwarz inequality, for any $n_1, n_2 > 0$,
\[
Vn = \left(\frac{G(S_1)^2}{n_1} + \frac{G(S_2)^2}{n_2}\right)(n_1 + n_2) \geq \left(G(S_1) + G(S_2) \right)^2.
\]
If we substitute $n_i = G(S_i)/Z$ for any $Z$ on the left of the inequality and simplify, we see the inequality becomes tight, so the minimum is achieved. We further require $\sum_i n_i = n$, so choose $Z$ so 
\[
n_i = n\cdot\frac{G(S_i)}{G(\Omega)} = n p(S_i).
\]
\end{proof}

\subsection{$k$-\discount variance}

Recall that $p(S) = \Pr[s_i \in S] = G(S)/G(\Omega)$ is the probability of a sample landing in $S$ under the sampling distribution $\bar g_{\Omega}$. 
Define
$$
\sigma^2(S) = \Var(f(s_i)/g(s_i) \mid s_i \in S) = \sum_{s \in S} \frac{g(s)}{G(S)} \cdot \left(\frac{f(s)}{g(s)} - \frac{F(S)}{G(S)}\right)^2.
$$ 
to be the variance of the importance weight for $s_i \sim \bar g_S$.

\begin{claim}
\label{claim:variance}
Let $r = 1 - p(S)$. The variance of the $k$-\discount estimator is given by
$$
\Var(\hat F_{\kDIS}(S)) = 
G(S)^2 \cdot \sigma^2(S) \cdot (1-r^n) \cdot \E\left[ \frac{1}{n(S)} \,\bigg\vert\, n(S) > 0\right] + F(S)^2 \cdot r^n\cdot(1-r^n).
$$
where
$(1-r)^n \E\left[ 1/n(S) \,\big\vert\, n(S) > 0\right] = \sum_{j=1}^n (1/j) \cdot \text{\normalfont Binomial}\big(j; n, p(S)\big)$.
\end{claim}

The second term in the variance arise from the possibility that no samples land in $S$; it decays exponentially in $n$ and is negligible compared to the first term.
The first term can be compared to the variance $G(S)^2 \cdot \sigma^2(S) \cdot \frac{1}{m}$ of importance sampling with exactly $m$ samples allocated to $S$ and the proposal distribution $\bar g_S$, i.e., \discount.
Because the sample size $n(S)$ is random, the correct scaling factor for $k$-\discount is $(1-r^n) \E[1/n(S) \mid n(S) > 0]$, which it turns out is asymptotically equivalent to $1/(np(S))$, i.e., \discount with a sample size of $m = n p(S) = \E[n(S)]$ --- see Claim~\ref{claim:asymptotic-variance} below. 
We find that for a small expected sample size (around 4) there can be up to 30\% ``excess variance'' due to the randomness in the number of samples (see Figure~\ref{fig:excess_variance}), but that this disappears quickly with larger expected sample size.
\begin{claim}
\label{claim:asymptotic-variance}
Let $\hat F_{\kDIS,n}$ and $\hat F_{\DIS,m}$ be the $k$-\discount and \discount estimators with sample sizes $n$ and $m$, respectively.
The asymptotic variance of $k$-\discount is given by
$$
\lim_{n \to \infty} n \Var(\hat F_{\kDIS,n}(S))  = G(S)^2 \cdot \sigma^2(S)/p(S).
$$
This is asymptotically equivalent to \discount with sample size $m = \E[n(S)] = n p(S)$.
That is
$$
\lim_{n \to \infty} \frac{\Var(\hat F_{\kDIS,n}(S))}{\Var(\hat F_{\DIS,\lceil n p(S) \rceil}(S))} = 1.
$$
\end{claim}

\begin{figure}
\centering
\includegraphics[width=0.6\textwidth]{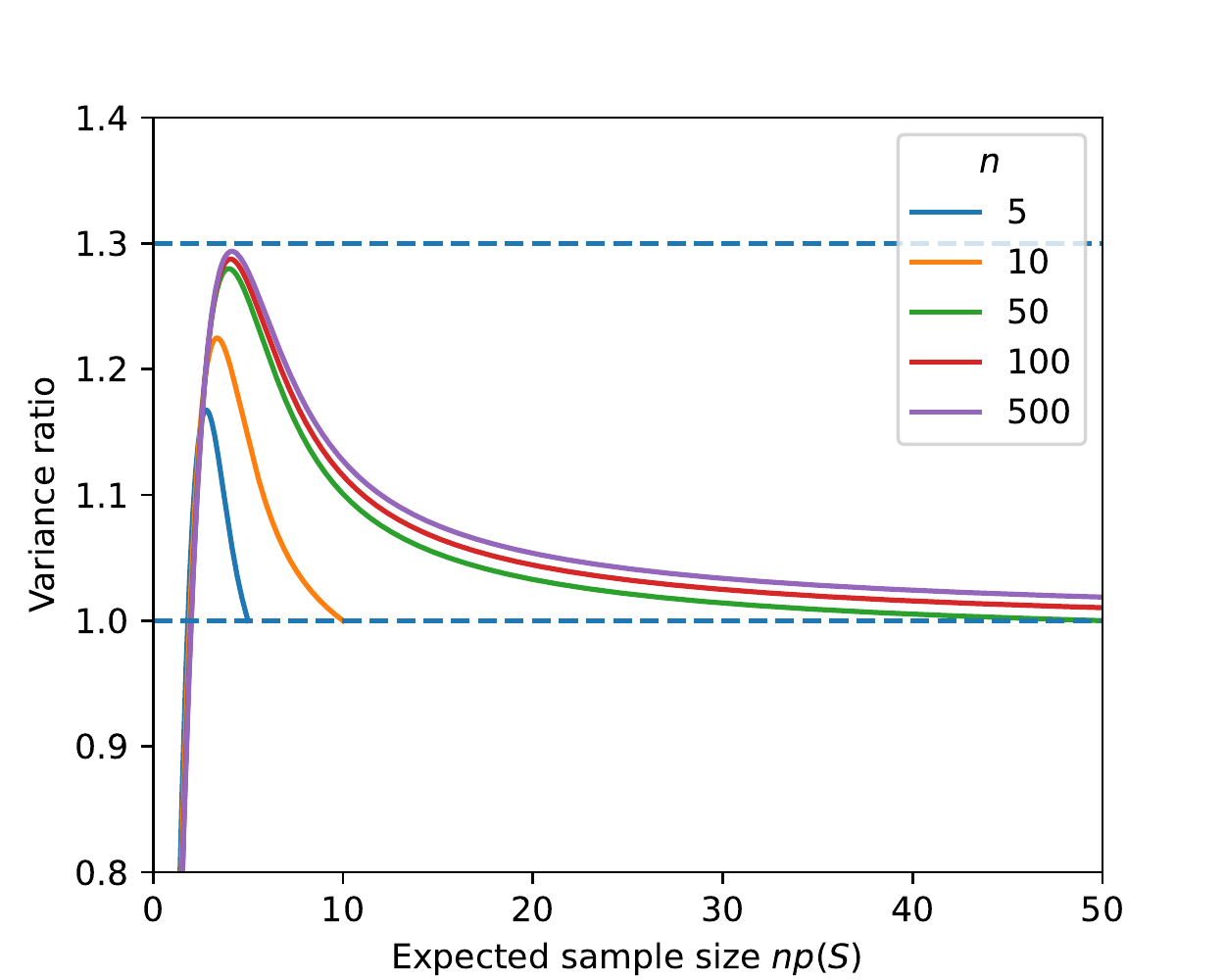}
\caption{Ratio of variance of $k$-\discount with expected sample size $n p(S)$ to \discount with $n p(S)$ samples; uses formula from Claim~\ref{claim:variance} (first term only).\label{fig:excess_variance}}
\end{figure}

\begin{proof}[Proof of Claim~\ref{claim:variance}]
By the law of total variance,
\begin{equation}
\label{eq:total-variance}
\Var(\hat F_{\kDIS}(S)) = \E \big[\Var(\hat F_{\kDIS}(S) \mid n(S))\big] + \Var \big(\E[\hat F_{\kDIS}(S) \mid n(S)]\big).
\end{equation}
We will treat each term in Eq.~\ref{eq:total-variance} separately. For the first term, from the definition of $k$-\discount we see
$$
\Var\big(\hat F_{\kDIS}(S) \mid n(S)) =
\begin{cases}    
0  & n(S)=0 \\
G(S)^2 \cdot \frac{\sigma^2(S)}{n(S)} & n(S) > 0
\end{cases}.
$$
Therefore
$$
\begin{aligned}
\E \big[\Var(\hat F_{\kDIS}(S) \mid n(S))\big] 
&= G(S)^2 \cdot \Pr[n(S) > 0] \E\left[ \frac{\sigma^2(S)}{n(S)} \,\bigg\vert\, n(S) > 0\right] \\
&= G(S)^2 \cdot \sigma^2(S) \cdot (1-r^n) \cdot \E\left[ \frac{1}{n(S)} \,\bigg\vert\, n(S) > 0\right].
\end{aligned}
$$
In the last line, we used the fact that $n(S) \sim \text{Binomial}(n, p(S))$, so $\Pr[n(S) > 0] = 1-r^n$ where $r = 1-p(S)$. The summation for $(1-r^n)\E[1/n(S) \mid n(S) > 0]$ follows from the same fact.

For the second term in Eq.~\eqref{eq:total-variance}, from the definition of $k$-\discount and conditional unbiasedness (Claim~\ref{claim:conditional-unbiasedness}), we have
$$
\begin{aligned}
\E[\hat F_{\kDIS}(S) \mid n(S)] &= 
\begin{cases}
0 & \text{ if }n(S)=0 \\[5pt]
F(S) & \text{ if } n(S) > 0
\end{cases} \\[5pt]
&= F(S) \cdot \text{Bernoulli}( 1-r^n).
\end{aligned}
$$
The variance is therefore
$$
\Var \big(\E[\hat F_{\kDIS}(S) \mid n(S)]\big) = F(S)^2 \cdot r^n \cdot (1-r^n).
$$
Putting the two terms together yields the result.
\end{proof}

\begin{proof}[Proof of Claim~\ref{claim:asymptotic-variance}]
By Claim~\ref{claim:variance} we have
\begin{multline}
\lim_{n \to \infty} 
n \Var(\hat F_{\kDIS,n}(S)) = 
\lim_{n \to \infty} n \cdot G(S)^2 \cdot \sigma^2(S) \cdot (1-r^n) \cdot \E\left[ \frac{1}{n(S)} \,\bigg\vert\, n(S) > 0\right] + \\
\lim_{n \to \infty} n \cdot F(S)^2 \cdot r^n\cdot(1-r^n).
\end{multline}
The second limit on the right side is zero, because $n r^n \to 0$ as $n \to \infty$ (recall that $r < 1$) and the other factors are bounded.
We will show the first limit on the right side is equal to $G(S)^2 \cdot \sigma^2(S) / p(S)$, which will prove the first part of the result.
The asymptotic expansion of \cite{znidarivc2009asymptotic} (Corollary 3) states that
$$
(1-r^n) \E[1/n(S) \mid n(S) > 0] = \frac{1}{n p(S)} + \mathcal O\left(\frac{1}{(n p(S))^2}\right).
$$
Using this expansion in the limit gives:
$$
\begin{aligned}
\lim_{n \to \infty} n \cdot G(S)^2 \cdot \sigma^2(S) \cdot (1-r^n) \cdot &\E\left[ \frac{1}{n(S)} \,\bigg\vert\, n(S) > 0\right]  \\
&= \lim_{n \to \infty} G(S)^2 \cdot \sigma^2(S) / p(S) \cdot ( 1 + \mathcal O(1/n) )\\
&= G(S)^2 \cdot \sigma^2(S) / p(S)
\end{aligned}
$$
The variance of $\discount$ with sample size $m$ is $\Var(\hat F_{\DIS, m}(S)) = G(S)^2 \cdot \sigma^2(S) / m$. 
Setting $m = \lceil n p(S)\rceil$ and using the second to last line above we have
$$
\begin{aligned}
\lim_{n \to \infty} \frac{\Var(\hat F_{\kDIS,n}(S))}{\Var(\hat F_{\DIS,\lceil n p(S) \rceil}(S))} 
&= \lim_{n \to \infty} \frac{n \Var(\hat F_{\kDIS,n}(S))}{n \Var(\hat F_{\DIS,\lceil n p(S) \rceil}(S))}  \\
&= \lim_{n \to \infty} \frac{G(S)^2 \cdot \sigma^2(S)/p(S) \cdot (1 + \mathcal O(1/n))}{n \cdot G(S)^2 \cdot \sigma^2(S) / \lceil n p(S) \rceil} \\
&= 1.
\end{aligned}
$$
\end{proof}

\subsection{Control Variates}

Recall that with control variates the weight is redefined as
\[
w_i = (f(s_i) - h(s_i))/g(s_i).
\]
The expectation of the weight given $s_i \in S$ is
\[
\E[w_i \mid s_i \in S] 
= \sum_{s \in S} \frac{g(s)}{G(S)} \frac{f(s) - h(s)}{g(s)} = \frac{F(S)-H(S)}{G(S)}
\]
Therefore
\[
\E[\bar w_{cv}(S) \mid n(S) > 0] = \frac{F(S)-H(S)}{G(S)} 
\]
Therefore
\[
\E[\hat F_{\kDIS cv}(S) \mid n(S) > 0] = G(S) \cdot \frac{F(S)-H(S)}{G(S)} + H(S) = F(S)
\]

\subsection{Confidence Intervals}

\begin{proof}[Proof of Claim~\ref{claim:asymptotic-distribution}]
Let $w_1, w_2, \ldots$ be an iid sequence of importance weights for samples in $S$, i.e., $w_i = f(s_i)/g(s_i)$ for $s_i \sim \bar g_S$. Each weight $w_i$ has mean $F(S)/G(S)$ and variance $\sigma^2(S)$. Let $\bar \omega_n = \frac{1}{n}\sum_{i=1}^n w_i$. By the central limit theorem,
$$
\sqrt{n}(\bar \omega_n - F(S)/G(S)) \toD \mathcal N(0, \sigma^2(S))
$$
Recall that $\hat F_{\kDIS,n}(S) = G(S) \cdot \bar w_n(S)$ where $\bar w_n(S)$ is the average of the importance weights for samples that land in $S$ when drawn from all of $\Omega$ (for clarity in the proof we add subscripts for sample size to all relevant quantities). It is easy to see that $\bar w_n(S)$ is equal in distribution to $\bar \omega_{n(S)}$ where $n(S) \sim \text{Binomial}(n, p(S))$ and $n(S)$ is independent of the sequence of importance weights --- this follows from first choosing the number of samples that land in $S$ and then choosing their locations conditioned on being in $S$. From Theorem 3.5.1 of \cite{van1996weak} (with $N_n = n(S)$ and $c_n = n p(S)$) it then follows that
$$
	\sqrt{n(S)}\big(\bar w_n(S) - F(S)/G(S)\big) \toD  \mathcal N(0, \sigma^2 (S))
$$
Rearranging yields
$$
\frac{\hat F_{\kDIS}(S) - F(S)}{G(S)/\sqrt{n(S)}} \toD  \mathcal N(0, \sigma^2 (S))
$$After dividing by $\hat \sigma_n(S)$, the result follows from Slutsky's lemma if $\hat \sigma^2_n(S) \toP \sigma^2(S)$, which follows from a similar application of Theorem 3.5.1 of \cite{van1996weak}.
\end{proof}

\section{Counting tasks}\label{app:countingtasks}

\begin{figure}[H]
    \centering
    \includegraphics[width=\linewidth]{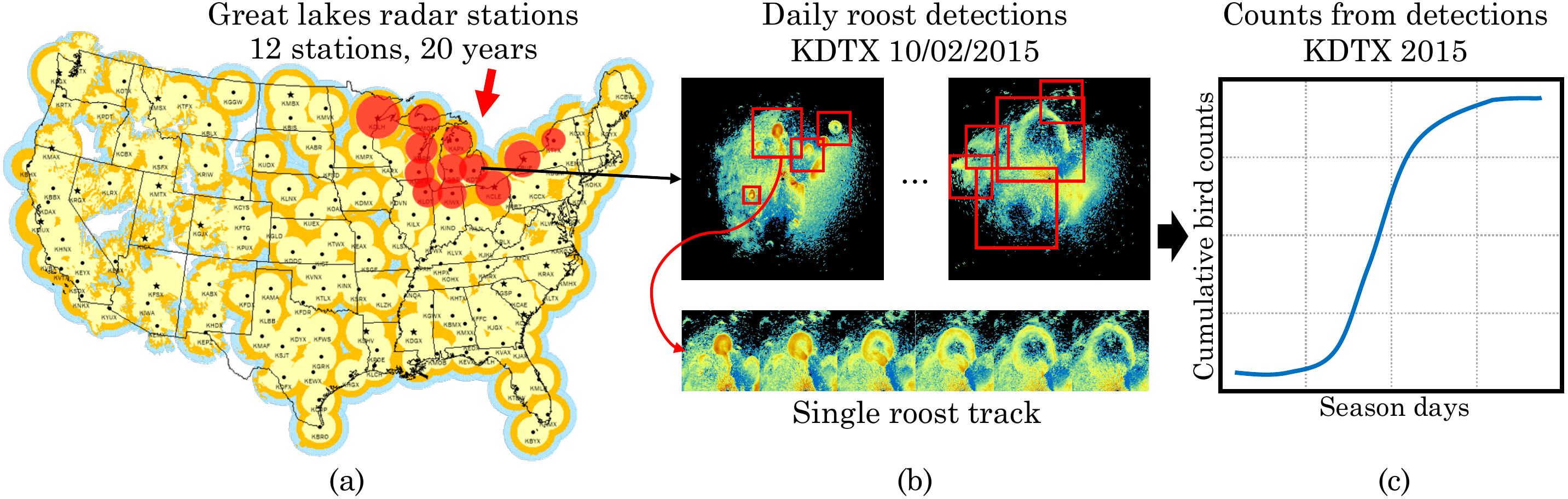}\\
    \vspace{10pt}
    \includegraphics[width=\linewidth]{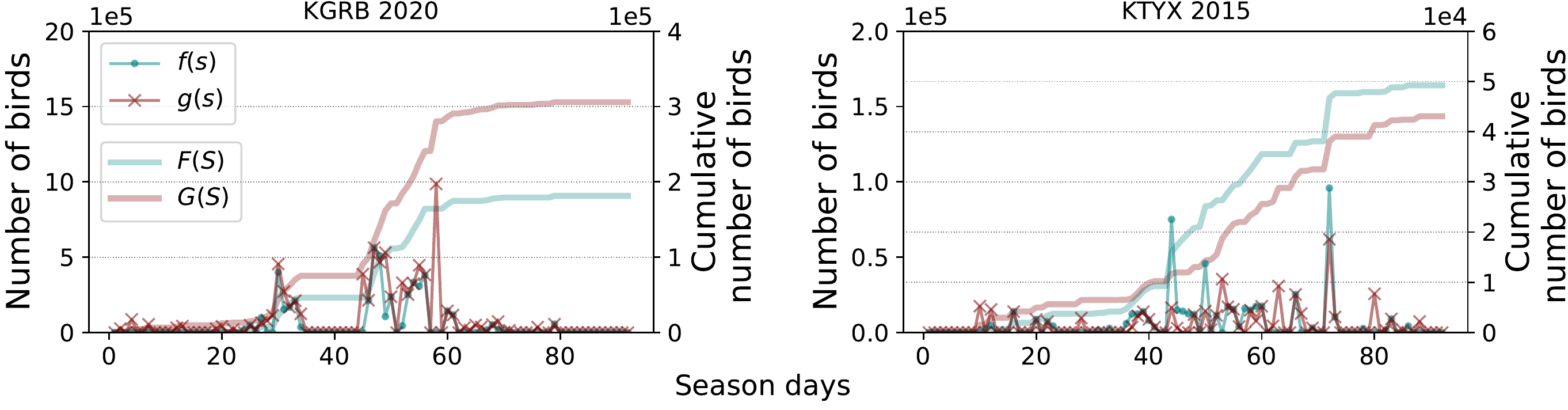}
    \vspace{-10pt}
    \caption{\textbf{Counting roosting birds in radar images.} \textbf{(a)} The US weather radar network has collected data for 30 years from 143+ stations and provides an unprecedented opportunity to study long-term and wide-scale biological phenomenon such as roosts. \textbf{(b)} Counts are collected for each day $s$ by running the detector using all radar scans for that day to detect and track roost signatures and then mapping to bird counts using the radar ``reflectivity'' within the tracks. The figure shows two scans for the KDTX station (Detroit, MI) on the same day, along roost detections which appear as expanding rings. By tracking these detections across a day one can estimate the number of birds in each roost. \textbf{(c)} Cumulative bird counts in the complete roosting season by aggregating counts across all tracked roosts and days. 
    \textbf{(bottom)} Examples of \emph{true} bird counts (blue) and detector counts (red) during a roosting season for station KGRB 2020 and KTYX 2015. 
    }
    \label{fig:radar}
\end{figure}

\begin{figure}[H]
    \centering
    \includegraphics[width=\linewidth]{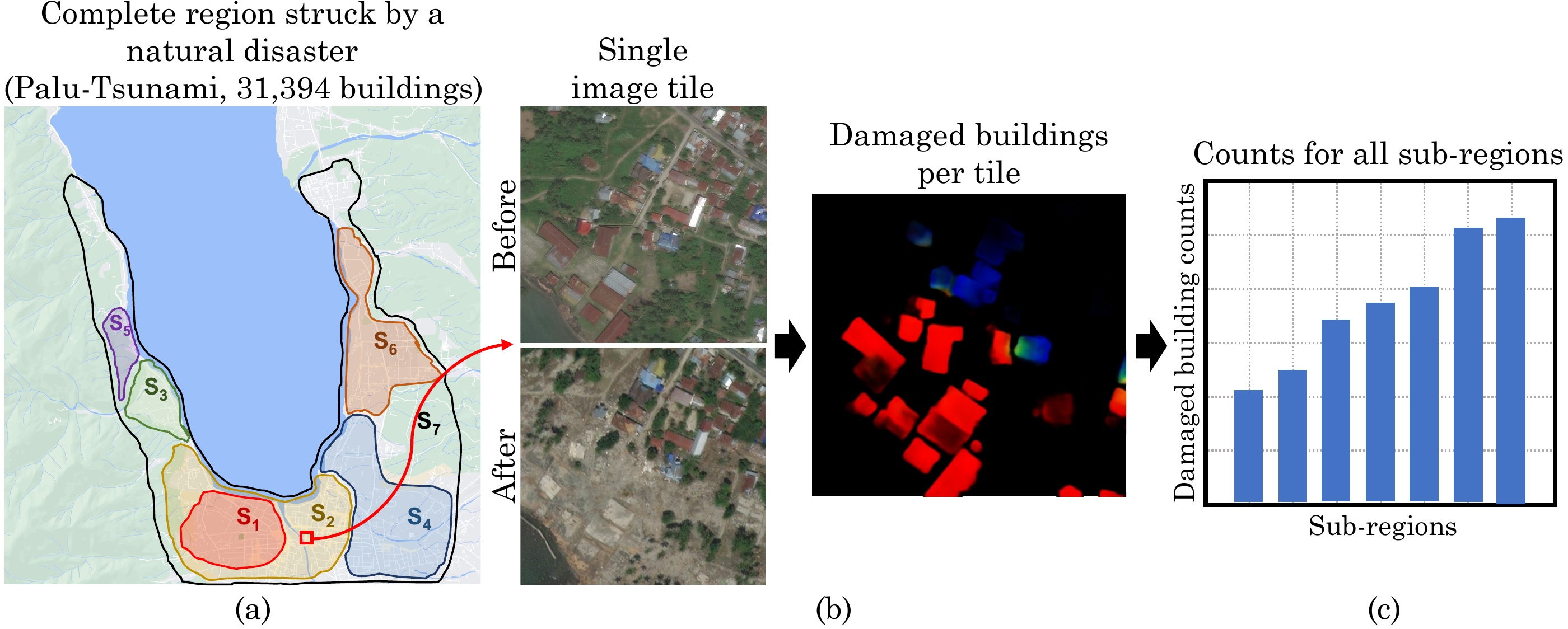}
    \vspace{-10pt}
    \caption{\textbf{Counting damaged buildings in satellite images.} Building damage assessment from satellite images~\cite{su2018,deng2022} is often used to plan humanitarian response after a natural disaster strikes. \textbf{(a)} We consider the Palu Tsunami from 2018; the data consists of 113 high-resolution satellite images labeled with 31,394 buildings and their damage levels. \textbf{(b)} Counts are collected per tile using before- and after-disaster satellite images. Colors indicate different levels of damage (e.g., red: "destroyed"). \textbf{(c)} Damaged building counts per sub-region. }
    \label{fig:app:buildings}
\end{figure}

\clearpage
\section{Palu Tsunami regions}\label{app:regions}

\begin{figure}[H]
    \centering
    \begin{tabular}[t]{c c}
    \includegraphics[width=0.385\linewidth]{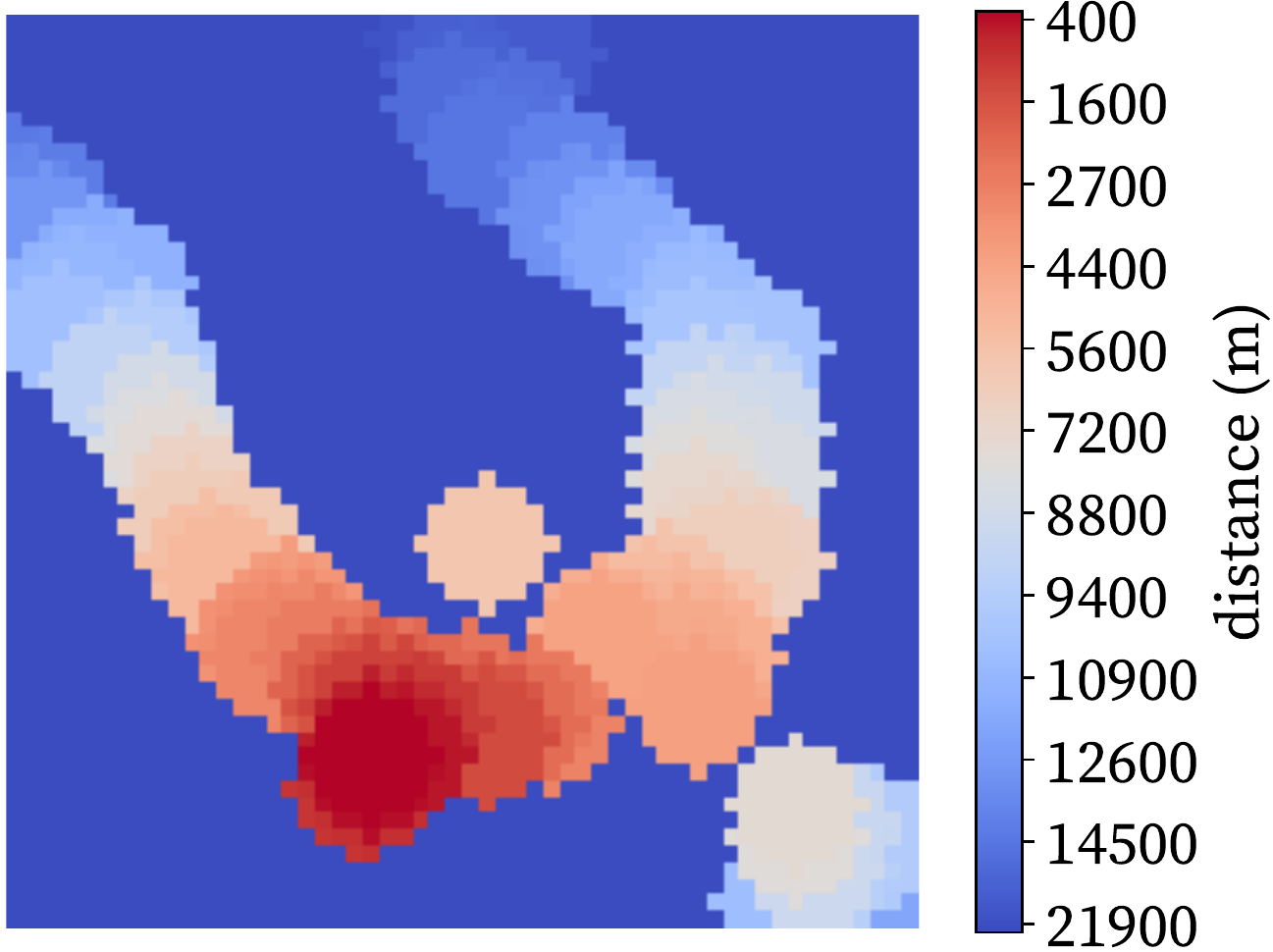} &
    \includegraphics[width=0.35\linewidth]{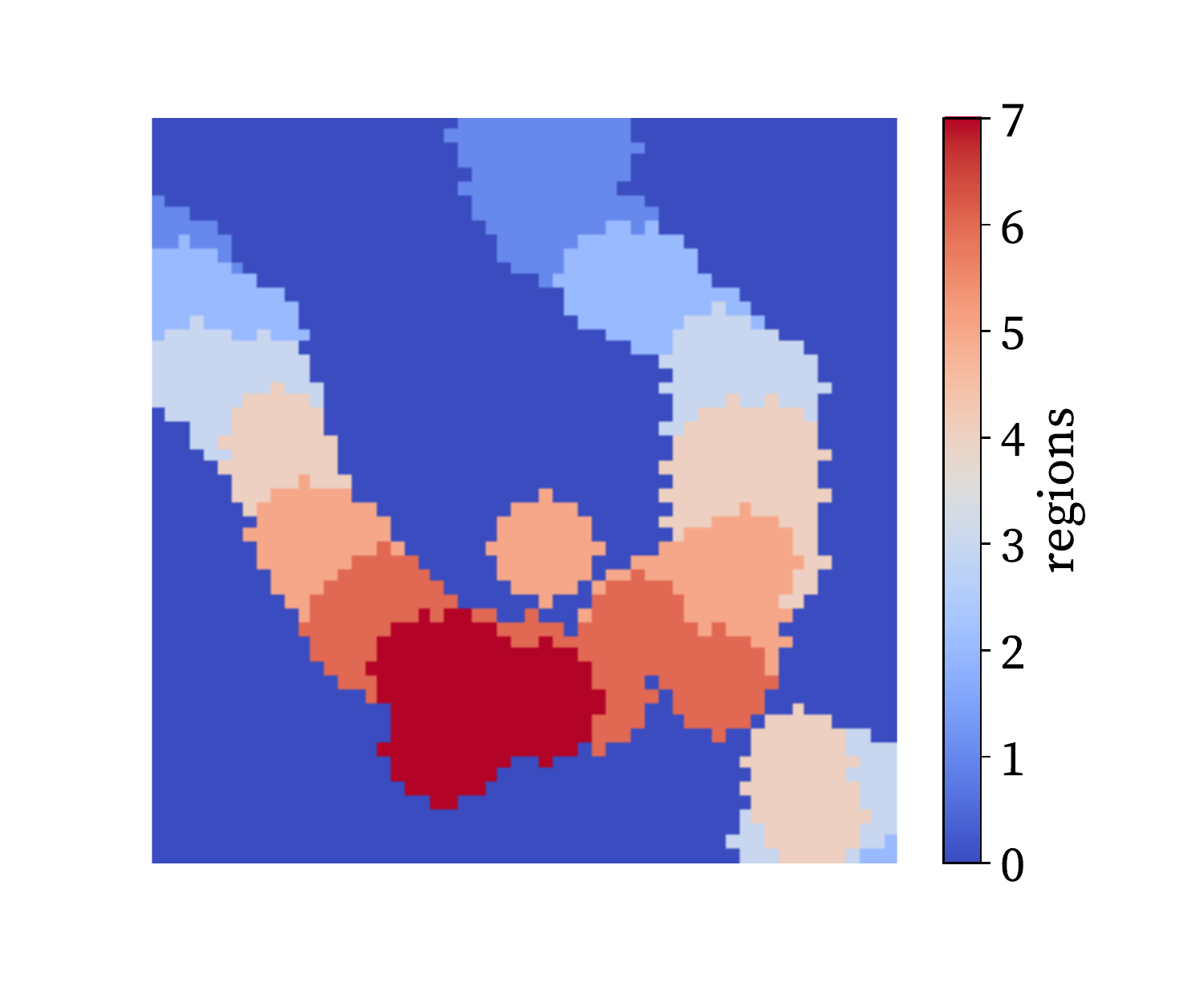} 
    \end{tabular}
    \caption{(left) Distance from the area with the most damaged buildings. (right) Regions defined for damaged building counting (See \S~\ref{sec:exp:buildings}). 
    }
    \label{fig:buildings_regions} 
\end{figure}

\end{document}